\documentclass[twoside]{article}
\usepackage{aistats2021}

\usepackage{nicefrac}
\usepackage{bbm}
\usepackage{bm}
\usepackage{microtype}
\usepackage{graphicx}
\usepackage{subfigure}
\usepackage{booktabs} 
\usepackage[numbers]{natbib}

\usepackage{amsthm}
\usepackage{amsmath}
\usepackage{color}
\usepackage{float}
\usepackage{amsfonts}
\usepackage{url}

\usepackage{hyperref}

\usepackage{bm}

\def\1{\bm{1}}

\def\vareps{{\varepsilon}}

\def\vzero{{\bm{0}}}

\def\vmu{{\bm{\mu}}}
\def\veps{{\bm{\eta}}}

\def\va{{\bm{a}}}

\def\vv{{\bm{v}}}
\def\vw{{\bm{w}}}
\def\vx{{\bm{x}}}
\def\vy{{\bm{y}}}

\def\mA{{\bm{A}}}
\def\mB{{\bm{B}}}
\def\mC{{\bm{C}}}

\def\mI{{\bm{I}}}

\def\mO{{\bm{O}}}

\def\mW{{\bm{W}}}
\def\mX{{\bm{X}}}

\def\mSigma{{\bm{\Sigma}}}

\DeclareMathAlphabet{\mathsfit}{\encodingdefault}{\sfdefault}{m}{sl}
\SetMathAlphabet{\mathsfit}{bold}{\encodingdefault}{\sfdefault}{bx}{n}

\def\gN{{\mathcal{N}}}

\def\gP{{\mathcal{P}}}
\def\gQ{{\mathcal{Q}}}
\def\gR{{\mathcal{R}}}

\def\gU{{\mathcal{U}}}

\def\gX{{\mathcal{X}}}

\def\gZ{{\mathcal{Z}}}

\def\sR{{\mathbb{R}}}

\def\emA{{A}}
\def\emB{{B}}
\def\emC{{C}}

\def\emV{{V}}

\def\emSigma{{\Sigma}}

\newcommand{\E}{\mathbb{E}}

\newcommand{\R}{\mathbb{R}}

\newcommand{\KL}{\mathbb{KL}}

\newcommand{\Cov}{\mathrm{Cov}}

\DeclareMathOperator{\sign}{sign}
\DeclareMathOperator{\Tr}{Tr}

\newcommand{\Ind}[1]{\ensuremath{\mathbbm{1} [#1]}}                     

\newcommand{\NormII}[1]{\ensuremath{\lVert #1 \rVert}_2}              

\newcommand{\InNorm}[1]{{\left\vert\kern-0.2ex\left\vert\kern-0.2ex\left\vert #1 
    \right\vert\kern-0.2ex\right\vert\kern-0.2ex\right\vert}}                    

\newcommand{\InNormII}[1]{{\left\vert\kern-0.2ex\left\vert\kern-0.2ex\left\vert #1 
    \right\vert\kern-0.2ex\right\vert\kern-0.2ex\right\vert}_2}                    

\newcommand{\InNormInfty}[1]{{\left\vert\kern-0.2ex\left\vert\kern-0.2ex\left\vert #1 
    \right\vert\kern-0.2ex\right\vert\kern-0.2ex\right\vert}_{\infty}}           

\newcommand{\Gauss}[2]{\mathcal{N}\left(#1,\; #2\right)}                         
\newcommand{\iid}{i.i.d.~}                                                        

\newtheorem{proposition}{Proposition}

\newtheorem{lemma}{Lemma}

\newtheorem{theorem}{Theorem}
\newtheorem{remark}{Remark}
\newtheorem{corollary}{Corollary}

\newcommand{\thetah}{\widehat{\theta}}

\newcommand{\TV}{\mathsf{TV}}

\DeclareMathOperator{\arccot}{\operatorname{arccot}}

\newcommand{\PF}{\mathcal{P}}
\newcommand{\QF}{\mathcal{Q}}

\begin{document}
%
%
%
%
\title{A Le Cam Type Bound for Adversarial Learning and Applications} 
\author{%
  Qiuling Xu\thanks{Equal contribution.}\\
  Department of Computer Science\\
  Purdue University\\
  West Lafayette, IN 47907, USA\\
  \texttt{xu1230@purdue.edu}
  \and
  Kevin Bello\footnotemark[1]\\
  Department of Computer Science\\
  Purdue University\\
  West Lafayette, IN 47907, USA\\
  \texttt{kbellome@purdue.edu}
  \and
  Jean Honorio\\
  Department of Computer Science\\
  Purdue University\\
  West Lafayette, IN 47907, USA\\
  \texttt{jhonorio@purdue.edu}
}
\date{}
\maketitle
\thispagestyle{empty} 
\begin{abstract}
Robustness of machine learning methods is essential for modern practical applications. 
Given the arms race between attack and defense mechanisms, it is essential to understand the fundamental limits of any conceivable learning method used in an adversarial setting.
In this work, we focus on the problem of learning from noise-injected data, where the existing literature falls short by either assuming a specific adversary model or by over-specifying the learning problem. 
We shed light on the information-theoretic limits of adversarial learning without assuming a particular adversary. 
Specifically, we derive a \textit{general} Le Cam type bound for learning from noise-injected data.
Finally, we apply our general bounds to a canonical set of non-trivial learning problems and provide examples of common types of noise-injected data.
\end{abstract}


\section{Introduction}

Modern machine learning models  can be vulnerable to different types of adversarial attacks.
This vulnerability affects a wide range of domains, including, computer vision \citep{DBLP:journals/corr/GoodfellowSS14, DBLP:conf/iclr/FischerKMB17}, video surveillance \citep{DBLP:conf/ndss/LiNPSKRS19}, natural language processing \citep{DBLP:conf/acl/EbrahimiRLD18}, voice recognition \citep{DBLP:conf/sp/Carlini018}, as well as different learning problems including classification \citep{DBLP:journals/corr/GoodfellowSS14}, regression \citep{DBLP:journals/tsp/BaldaBM19}, node embeddings \citep{DBLP:conf/icml/BojchevskiG19}, generative models \citep{DBLP:conf/eurosp/PasquiniMB19}, among others.

In the literature, one can differentiate two general lines of work regarding the type of data (training or testing) an adversary can attack.
Perhaps the most widely considered type of attack is at test time, where some important theoretical prior work in this regime includes 
generalization bounds for adversarial learning~\citep{conf/icml/YinRB19, attias2018improved}, 
the study of different notions of robustness certification for inference~\citep{conf/icml/CohenRK19,hein2017formal,random_smooth2}, 
robustly PAC learnability of VC classes \citep{montasser2019vc},
and analysis of the effect of injecting noise in the network at inference time \citep{random_smooth1}.
While important, generalization in the aforementioned works does not quantify the effect of adversarial attacks at the learning stage
and, thus, disregards the statistical challenges brought up by learning from noise-injected data.
Tackling the issues mentioned above, the work of~\cite{sinha2017certifying} provides bounds for learning from noise-injected data, versus learning from the ``original'' data distribution.

A key question that has not been addressed in the prior literature is the study of the fundamental \textit{statistical limits} of adversarial robustness at the \textit{learning stage}, for which we refer to as \textit{adversarial learning}.
That is, how much is a learner affected when using noise-injected data, instead of having access to the original data distribution.
To the best of our knowledge, the work of \cite{huber1992robust} is among the first to formalize this type of question through an $\epsilon$-contamination model, in which an $\epsilon$ fraction of observations are subject to arbitrary adversarial noise.
Lately, the $\epsilon$-contamination model has received some attention from the community, for which authors in \cite{du2018robust} develop algorithms for robust nonparametric regression, and the work of \cite{chen2016general} derives a general decision theory.
In this work, while we consider the question of noise-injected data at learning time, we take the approach of assuming that the attacker has a ``budget'' of how much noise is injected to the data.
We relate this budget to the \textit{total variation (TV) distance} between the \textit{original data distribution} and the \textit{noise-injected data distribution}.
Our choice of using the TV distance is motivated by the fact that the TV distance is a core statistical distance between probability measures and has several connections to other distances (see, e.g., \cite{gibbs2002choosing}). 
Also, the TV distance has connections to upper and lower bounds in the study of adversarial robustness (see, e.g., \cite{bhagoji2019lower, thekumparampil2018robustness, random_smooth1}), thus, we consider the TV to be a general and flexible notion of distance.
Finally, a similar assumption, albeit for generalization bounds, was made in~\cite{sinha2017certifying} which used the Wasserstein distance.

Our main contributions are as follows: 

1. To the best of our knowledge, we are the first to provide a general information-theoretic lower bound for the problem of learning from noise-injected data.
Our goal is not to specify a particular algorithm modeling the adversary.
Instead, we assume that the adversary has a ``budget'' in total variation distance between the original data distribution and the noise-injected data distribution.

2. To this end, we extend Le Cam's bound to the adversarial learning regime.
Our general framework allows for the analysis of learning problems and noise injection methods in a somewhat independent fashion.

3. We illustrate our general framework by applying it on three canonical learning problems: mean estimation, binary classification and Procrustes analysis.
In addition, we analyze two types of noise injection methods, by using either a (multivariate) Gaussian or a (multivariate) uniform noise.
By combining the above results, we deliver lower bounds for six different scenarios in total.
\section{A Brief Review of the General Minimax Risk Framework and Le Cam's Lemma}
\label{sec:minimax_framework}
	In this section we briefly review the minimax framework in the context of general statistical problems.
	The minimax framework consists of a well defined objective that aims to shed light about the optimality of algorithms and has been widely used in statistics and machine learning \citep{wainwright2019high,wasserman2006all}.
	The standard minimax risk considers a family of distributions $\gQ$ over a sample space $\gZ$, and a function $\theta : \gQ \to \Theta$ defined on $\gQ$, that is, a mapping $Q \mapsto \theta(Q)$.
	Here we call $\theta(Q)$ the parameter of the distribution $Q$.
	We aim to estimate the parameter $\theta(Q)$ based on a sequence of $n$ \iid observations $Z = (z_1,\ldots,z_n)$ drawn from the (unknown) distribution $Q$, that is, $Z \in \gZ^n$.
	To evaluate the quality of an estimator $\thetah$, we let $\rho : \Theta \times \Theta \to \sR_+$ denote a semi-metric on the space $\Theta$, which we use to measure the error of an estimator $\thetah$ with respect to the parameter $\theta(Q)$.
	For a distribution $Q \in \gQ$ and for a given estimator $\thetah : \gZ^n \to \Theta$, we assess the quality of the estimate $\thetah(Z)$ in terms of the (expected) risk:
	\[
		\E_{Z \sim Q^n} \left[ \rho(\thetah(Z), \theta(Q)) \right].
	\]
	A common approach, first suggested by \citep{wald1939contributions}, for choosing an estimator $\thetah$ is to select the one that minimizes the maximum risk, that is,
	\[
		\sup_{Q \in \gQ} \E_{Z \sim Q^n} \left[ \rho(\thetah(Z), \theta(Q)) \right].
	\]
	An optimal estimator for this semi-metric then gives the minimax risk, which is defined as:
	\begin{align}
	\label{eq:standard_minimax}
		\inf_{\thetah} \sup_{Q \in \gQ} \E_{Z\sim Q^n} \left[ \rho( \thetah (Z), \theta(Q) ) \right],
	\end{align}
	where we take the supremum (worst-case) over distributions $Q \in \gQ$, and the infimum is taken over all estimators $\thetah$.
	
	There are several common approaches to find lower bounds to eq.\eqref{eq:standard_minimax} such as using Fano's, Le Cam's or Assouad's bounds \citep{Yu97}.
	Next, we present a version of Le Cam's bound, which will be later adapted to our adversarial setting.
	\begin{lemma}[Le Cam's bound \citep{Wasserman10}]
	\label{lemma:lecam}
	Let $\gP$ be a family of distributions over the sample space $\gX$.
	For a distribution $P$, let $\theta(P)$ denote a parameter of $P$. 
	Also, let $\thetah: \gX^n \to \Theta$ denote any parameter estimator that receives $n$ \iid samples $S = (x_i, \ldots, x_n)$ coming from some product distribution $P^n$ and outputs an estimate, $\thetah(P)$, of $\theta(P)$.
	Let $d: \Theta \times \Theta \rightarrow \sR_+$ denote a metric\footnote{While it is know that one can relax this assumption by using a pseudo-metric that fulfills the weak triangle inequality (see, for instance, \cite{Yu97}), we use a metric for clarity purposes.} on the space of parameters $\Theta$. 
	Then, for any $P_1, P_2 \in \gP$, we have:
	\begin{align*}
		&\inf_{\thetah} \sup_{P \in \gP} \E_{S \sim P^n} \left[ d(\thetah(P), \theta(P)) \right]  \\
		&\geq \frac{d(\theta(P_1),\theta(P_2))}{4}  \int_{S \in \gX^n}  \min(p_1^n(S), p_2^n(S)) dS.
	\end{align*}
	\end{lemma}
	
\section{On the Statistical Limits of Adversarial Learning}
	
	In this section, we explain in more detail our framework under consideration, for which we derive information-theoretic bounds.
	Before stating our main results, we present the following proposition that is used for the proof of Theorem \ref{thm:minimax_2_distributions}.
	\begin{proposition}[\cite{Wasserman10}]
	\label{lemma:pearson_bound}
	Let $P_1$ and $P_2$ be two distributions with support on $\gX$, and let $S \sim P^n$ denote a collection of $n$ \iid samples drawn from some distribution $P$.
	For any function, $\Psi: \gX^n \to \{0,1\}$, we have:
	\begin{align*}
	 &\hspace{-0.2in}\Pr_{S\sim P_1^n} \left[\Psi(S) = 1\right] + \Pr_{S\sim P_2^n}\left[\Psi(S) = 0\right]\\
	 &\hspace{0.5in} \geq \int_{S \in \gX^n} \min\left(p_1^n(S), p_2^n(S)\right) dS.
	\end{align*}
	\end{proposition}
	Consider a distribution $Q$, from a family of distributions $\gQ$ with support $\gX$, from which we draw a sample set $S$, and we aim to estimate a parameter of \textit{another} distribution $P$ coming from a family $\gP$ over the same support $\gX$.
	In order for this to make sense, one has to consider the families $\gP$ and $\gQ$ to be somehow related.
	In our particular case, $\gP$ corresponds to the family of ``original'' distributions for which we would like to have good parameter estimates; while $\gQ$ corresponds to the family of ``adversarial'' (noise-injected) distributions from which we actually observe the dataset $S$.
	The following theorem is our first result and corresponds to an extension of the Le Cam's bound to our adversarial learning setting.
	\begin{theorem}
	\label{thm:minimax_2_distributions}
	Let $\gP$ and $\gQ$ be two families of distributions over the sample space $\gX$.
	For a distribution $P$, let $\theta(P)$ denote a parameter of $P$ that we aim to estimate. 
	Also, let $\thetah: \gX^n \to \Theta$ denote any learning method that receives $n$ \iid samples $S = (x_1, \ldots, x_n)$ coming from some product distribution $Q^n$ and outputs an estimate, $\thetah(S)$, of $\theta(P)$.
	Let $d: \Theta \times \Theta \rightarrow \sR_+$ denote a metric on the space of parameters $\Theta$. 
	Then, for any $P_1, P_2 \in \gP$, we have:
	\begin{align*}
	    &\inf_{\thetah} \sup_{P \in \gP, Q \in \gQ} \E_{S \sim Q^n} \left[ d(\thetah(S), \theta(P)) \right]  \\
	    &\geq \frac{d(\theta(P_1),\theta(P_2))}{4} \sup_{Q_1,Q_2\in \QF} \int_{S \in \gX^n} \hspace{-0.1in} \min(q_1^n(S), q_2^n(S)) dS.
	\end{align*}
	\end{theorem}
	\begin{proof}
	\refstepcounter{equation}
	Let \( \Psi(S) = \Ind{d(\thetah(S),\theta(P_2)) \leq d(\thetah(S),\theta(P_1)) } \).
	If $\Psi(S) = 1$, we have:
	\begin{align*}
		d(\theta(P_1),\theta(P_2)) &\leq d(\thetah(S),\theta(P_1)) + d(\thetah(S),\theta(P_2)) \\
		&\leq 2\cdot d(\thetah(S),\theta(P_1)),
	\end{align*}
	where the first inequality above is due to $d$ being a metric.
	Thus,
	\begin{align}
	\mathop{\E}_{S\sim Q^n} \hspace{-0.05in} [d(\thetah(S),\theta(P_1))] 
	&\geq \hspace{-0.05in}\mathop{\E}_{S\sim Q^n} \hspace{-0.05in} [ d(\thetah(S),\theta(P_1)) \Ind{\Psi(S) = 1}]  \nonumber\\
	&\hspace{-0.7in}\geq \frac{d(\theta(P_1),\theta(P_2))}{2}  \Pr_{S\sim Q^n}[\Psi(S) = 1]. \tag{\theequation.a} \label{eq:thm1_bound1}
	\end{align}
	Similarly, if $\Psi(S) = 0$, we have:
	\begin{align*}
		d(\theta(P_1),\theta(P_2)) &\leq d(\thetah(S),\theta(P_1)) + d(\thetah(S),\theta(P_2)) \\
		&\leq 2\cdot d(\thetah(S),\theta(P_2)).
	\end{align*}
	Thus,
	\begin{align}
	\hspace{-0.05in}\mathop{\E}_{S\sim Q^n} \hspace{-0.05in} [d(\thetah(S),\theta(P_2))] 
	&\geq \hspace{-0.05in}\mathop{\E}_{S\sim Q^n} [ d(\thetah(S),\theta(P_2)) \Ind{\Psi(S) = 0}]  \nonumber\\
	&\hspace{-0.7in} \geq \frac{d(\theta(P_1),\theta(P_2))}{2} \Pr_{S\sim Q^n}[\Psi(S) = 0]. \tag{\theequation.b} \label{eq:thm1_bound2}
	\end{align}
	Combining eq.\eqref{eq:thm1_bound1} and eq.\eqref{eq:thm1_bound2}, for any estimator $\thetah$, we have:
	\begin{align}
	&\sup_{\substack{P \in \gP \\ Q \in \gQ}} \mathop{\E}_{S \sim Q^n} [d(\thetah(S), \theta(P))]  \nonumber\\
	&\geq \max \bigg(  \sup_{Q\in \gQ} \mathop{\E}_{S \sim Q^n} [d(\thetah(S), \theta(P_1))], \nonumber \\
	&\hspace{0.6in} \sup_{Q\in \gQ} \mathop{\E}_{S \sim Q^n} [d(\thetah(S), \theta(P_2))] \bigg) \notag \\
	&\geq \frac{d(\theta(P_1),\theta(P_2))}{2} \max \bigg( \sup_{Q\in \gQ} \Pr_{S\sim Q^n}[\Psi(S) = 1], \notag\\
	&\hspace{1.5in}\sup_{Q\in \gQ} \Pr_{S\sim Q^n}[\Psi(S) = 0] \bigg) \notag \\
	&\geq \frac{d(\theta(P_1),\theta(P_2))}{4} \bigg( \sup_{Q\in \gQ} \Pr_{S\sim Q^n}[\Psi(S) = 1] \notag \\
	&\hspace{1.1in}+ \sup_{Q\in \gQ} \Pr_{S\sim Q^n}[\Psi(S) = 0] \bigg) \tag{\theequation.c} \label{eq:sum_id} \\
	&= \frac{d(\theta(P_1),\theta(P_2))}{4}  \sup_{Q_1,Q_2 \in \gQ} \Big( \Pr_{S\sim Q_1^n}[\Psi(S) = 1] \notag \\
	&\hspace{1.6in}+  \Pr_{S\sim Q_2^n}[\Psi(S) = 0] \Big)  \notag \\
	&\geq \frac{d(\theta(P_1),\theta(P_2))}{4} \sup_{Q_1,Q_2\in \QF} \int_{S \in \gX^n}  \hspace{-0.1in} \min(q_1^n(S),q_2^n(S)) dS,\notag
	\end{align}
	where eq.\eqref{eq:sum_id} follows from $\max(a,b) \geq \frac{a+b}{2}$, and the last inequality is due to Proposition \ref{lemma:pearson_bound}.
	\end{proof}    
	From the above argument, it is reasonable to consider the following setting.
	Given $P \in \gP$ and $\beta \in [0,1]$, we define the family $\gQ(P,\beta) = \{Q \mid \TV(Q,P) \leq \beta \}$.
	That is, the set $\gQ(P,\beta)$ contains distributions $Q$ that are at most $\beta$-away to $P$ with respect to the total variation distance.
	Thus, we analyze the case in which the adversary chooses a distribution $Q \in \gQ(P,\beta)$ and we observe data from it.
	The following theorem considers the aforementioned setting.
	\begin{theorem}
	\label{col:adversarial_minimax}
	Under the same setting of Theorem \ref{thm:minimax_2_distributions}, and letting $\gQ(P,\beta) = \{Q \mid \TV(Q,P) \leq \beta\}$ for each $P \in \gP$ and some $\beta \in [0,1]$.
	We have, for any $P_1, P_2 \in \gP$,
	\begin{align*}
	    &\inf_{\thetah} \sup_{\substack{P \in \gP \\ Q \in \gQ(P,\beta)}} \E_{S \sim Q^n} \left[ d(\thetah(S), \theta(P)) \right]  \\
	    &\geq \frac{d(\theta(P_1),\theta(P_2))}{4} \left(\int_{S \in \gX^n} \hspace{-0.1in} \min(p_1^n(S),p_2^n(S)) dS + \delta \right), 
	\end{align*}
	where $\displaystyle 0 \leq  \delta \leq \sup_{\substack{Q_1 \in \gQ(P_1,\beta) \\ Q_2 \in \gQ(P_2,\beta)}} \TV(P_1^n,Q_1^n) + \TV(P_2^n,Q_2^n)$, and $\TV$ is the total variation distance.
	\end{theorem}
	\begin{proof}
	\refstepcounter{equation}
	Under similar arguments for the proof of Theorem \ref{thm:minimax_2_distributions} up to eq.\eqref{eq:sum_id}, we have that for any estimator $\thetah$:
	\begin{align}
	&\sup_{\substack{P \in \gP \\ Q \in \gQ(P,\beta)}} \mathop{\E}_{S \sim Q^n} [d(\thetah(S), \theta(P))]  \geq \frac{d(\theta(P_1),\theta(P_2))}{4} \Big(  \notag\\
	&\sup_{Q \in \gQ(P_1,\beta)} \Pr_{S\sim Q^n}[\Psi(S) = 1] + \hspace{-0.1in}\sup_{Q\in \gQ(P_2,\beta)} \Pr_{S\sim Q^n}[\Psi(S) = 0] \Big) \tag{\theequation.a} \label{eq:sum_id2}
	\end{align}	
	Let $ \displaystyle \delta_1 = \hspace{-0.05in}\sup_{Q\in \gQ(P_1,\beta)} \Big(\Pr_{S\sim Q^n}[\Psi(S) = 1] - \hspace{-0.05in}\Pr_{S\sim P_1^n}[\Psi(S) = 1] \Big)$.
	Then, for $\sup_{Q\in \gQ(P_1,\beta)} \Pr_{S\sim Q^n}[\Psi(S) = 1]$ we have:
	\begin{align}
	\sup_{Q \in \gQ(P_1,\beta)} \Pr_{S\sim Q^n} \left[\Psi(S) = 1 \right] = \Pr_{S\sim P_1^n} \left[\Psi(S) = 1 \right] + \delta_1 \tag{\theequation.b} \label{eq:col1_bound1},
	\end{align}
	where $\delta_1$ is bounded as follows:
	\begin{align}
	|\delta_1| &= \left| \sup_{Q\in \gQ(P_1,\beta)} \left( \Pr_{S\sim Q^n}\hspace{-0.02in}[\Psi(S) = 1] - \hspace{-0.07in} \Pr_{S\sim P_1^n}[\Psi(S) = 1]\right) \right| \notag \\
	&\leq \sup_{Q\in \gQ(P_1,\beta)} \left| \Pr_{S\sim Q^n}[\Psi(S) = 1] - \Pr_{S\sim P_1^n}[\Psi(S) = 1] \right| \notag\\
	&\leq \sup_{Q \in \gQ(P_1,\beta)} \TV(Q^n,P_1^n). \notag
	\end{align}
	Also, we have that \(0 \leq \delta_1 \leq \sup_{Q \in \gQ(P_1,\beta)} \TV(Q^n,P_1^n)\) since $P_1 \in \gQ(P_1,\beta)$.
	\\Similarly, for $\sup_{Q\in \gQ(P_2,\beta)} \Pr_{S\sim Q^n}[\Psi(S) = 0]$, let $\displaystyle \delta_2 = \sup_{Q\in \gQ(P_2,\beta)} \left(\Pr_{S\sim Q^n}[\Psi(S) = 0] - \Pr_{S\sim P_2^n}[\Psi(S) = 0]\right)$, we have:
	\begin{align}
	\sup_{Q\in \gQ(P_2,\beta)} \Pr_{S\sim Q^n}[\Psi(S) = 0] = \Pr_{S\sim P_2^n}[\Psi(S) = 0] + \delta_2 \tag{\theequation.c} \label{eq:col1_bound2},
	\end{align}
	where $0 \leq \delta_2 \leq \sup_{Q \in \gQ(P_2,\beta)} \TV(Q^n,P_2^n).$
	Combining eq.\eqref{eq:col1_bound1} and eq.\eqref{eq:col1_bound2} with eq.\eqref{eq:sum_id2}, and letting $\delta = \delta_1 + \delta_2$, we have:
	\begin{align}
	&\sup_{\substack{P \in \gP \\ Q \in \gQ(P,\beta)}} \E_{S \sim Q^n} [d(\thetah(S), \theta(P))]  \geq \frac{d(\theta(P_1),\theta(P_2))}{4} \bigg( \notag \\ 
	&\hspace{0.5in}\Pr_{S\sim P_1^n}[\Psi(S) = 1] + \Pr_{S\sim P_2^n}[\Psi(S) = 0] + \delta \bigg)  \notag \\
	&\geq \frac{d(\theta(P_1),\theta(P_2))}{4} \left(\int_{S \in \gX^n} \hspace{-0.1in} \min(p_1^n(S),p_2^n(S)) dS + \delta \right), \notag
	\end{align}	
	where $\displaystyle 0 \leq \delta \leq \hspace{-0.1in}\sup_{\substack{Q_1 \in \gQ(P_1,\beta) \\ Q_2 \in \gQ(P_2,\beta)}} \hspace{-0.1in} \TV(Q_1^n,P_1^n) + \TV(Q_2^n,P_2^n).$
	\end{proof}
	
	\begin{remark}
	\label{remark:different_Q}
	In Theorems \ref{thm:minimax_2_distributions} and \ref{col:adversarial_minimax}, we used the product distribution $Q^n$ for clarity purposes.
	However, our result also applies for the case when the $i$-th sample comes from a distribution $Q_i$, i.e., $S \sim Q_1 \times \ldots \times Q_n$.
	For instance, one can model each $Q_i$ to have different means and variances.
	It will become clear in our examples that those extensions are trivial.
	\end{remark}
	
	\begin{remark}
	\label{remark:bound}
	Note that under the setting of Theorem \ref{col:adversarial_minimax}, if there is no adversary, i.e., $\beta = 0$, then $\delta = 0$ and the lower bound reduces to the Le Cam bound from Lemma \ref{lemma:lecam}.
	Thus, from that viewpoint, our bound is tight.
	Also, observe that $\delta \leq 2\beta$ by the definition of the set $\gQ(P,\beta)$, that is, we would pay \emph{at most} $2\beta$ in the minimax bound under this adversarial framework.
	\end{remark}
	
	The reader should also note that it is possible that a set $\gQ(P,\beta)$ does not contain a distribution $Q$ such that $\TV(Q,P) = \beta$. 
	This is the reason why in Theorem \ref{col:adversarial_minimax} we leave $\delta$ expressed in an interval.
	Finally, we highlight the appealing decoupling property of Theorem \ref{col:adversarial_minimax}, by comparing it to the Le Cam bound in Lemma \ref{lemma:lecam}, we note that $\frac{d(\theta(P_1),\theta(P_2))}{4} \delta$ is the only extra term, which implies that we can use existing applications of Le Cam bounds in the literature and only analyze the extra adversarial term $\delta$.
	In the next section, we show the applicability of Theorem \ref{col:adversarial_minimax} through seemingly different problems (mean estimation, classification, and Procrustes analysis), as well as examples of adversarial noise (multivariate Gaussian, and multivariate uniform).
\section{Applications}
In this section, we show examples of our adversarial lower bounds in different canonical learning settings. 
Benefited from the decomposability of the lower bound in Theorem \ref{col:adversarial_minimax} into the Le Cam bound and the adversarial term, we can tackle each term separately and later combine them together.

Since it is impossible to directly model all possible distributions, we apply similar techniques used in several applications of Le Cam's bound \citep{wainwright2019high}.
The idea is to define a family of distributions parameterized by some variables, e.g., a 1-dimensional Gaussian family with constant variance and where the mean is allowed to vary over the reals.
Then, clearly the best any estimator can do is to be as close as possible to the true mean of a given distribution from the family.

As an illustration, we first study three canonical learning problems in Section \ref{learning_probs}: Mean estimation, binary classification and Procrustes analysis. 
Afterwards, in Section \ref{gaussian}, we provide upper bounds on the adversarial term $\delta$.
The following table summarizes our results from the next sections.

\begin{table}[]
\centering
\caption{Summary of our results. 
We obtain three lower bounds from three different learning problems.
Also, we provide two upper bounds of $\delta$ from two different noise distributions.
Combining the results on learning problems and noise distributions, we can then obtain 6 bounds in total.
More details are specified more formally in their respective lemmas.}
\label{tab:summary}
\begin{tabular}{@{}cc@{}}
\toprule
\textbf{Learning problem}                                                                                 & \textbf{Lower Bound}                                                                                                                           \\ \midrule
\begin{tabular}[c]{@{}c@{}}Mean estimation\\ (Lemma \ref{lemma:mean_estimation})\end{tabular}             & $\frac{\sqrt{\lambda_{\min}}}{8\sqrt{n}} \left(\frac{1}{\sqrt{e}} + 2\delta \right)$                                                     \\
\begin{tabular}[c]{@{}c@{}}Binary classification\\ (Lemma \ref{lemma:binary_classification})\end{tabular} & $\frac{\lambda_{\min}}{8n} \left(\frac{1}{\sqrt{e}} + 2\delta \right)$                                                                   \\
\begin{tabular}[c]{@{}c@{}}Procrustes analysis\\ (Lemma \ref{lemma:linear_regression})\end{tabular}       & $\frac{\epsilon^2}{8n\sigma^2} \left(\frac{1}{e} + 2\delta \right)$                                                                      \\ \midrule
\midrule
\textbf{Noise distribution}                                                                               & \textbf{Upper Bound on $\delta$}                                                                                                                           \\ \midrule
\begin{tabular}[c]{@{}c@{}}Multivariate Gaussian\\ (Lemma \ref{lemma:mulvar_gaussian_noise})\end{tabular}              & $\frac{c\sqrt{n}}{2\sqrt{\lambda_{\min}}}$                                                                                               \\
\begin{tabular}[c]{@{}c@{}}Multivariate uniform\\ (Lemma \ref{lemma:mulvar_uniform_noise_2})\end{tabular}              & \scriptsize$\sqrt{ c \sum_{i=1}^k \hspace{-0.025in} \frac{n}{\sqrt{ 2\pi \emSigma_{ii}}}  + \hspace{-0.025in} \frac{nc^2}{4}\Tr(\mB\mSigma^{-1}) }$ \\ \bottomrule
\end{tabular}
\end{table}

\subsection{Canonical learning problems}
\label{learning_probs}
The following lemma corresponds to the classical task of mean estimation.
In classical lower bounds \citep{wainwright2019high}, the task typically consists of a Gaussian distribution $P$, and a set of samples $S$ coming from $P$ that is used to estimate the mean.
We emphasize the difference that, in our setting, the observation $S$ comes from a \textit{poisoned} distribution $Q$, which possibly is no longer a Gaussian distribution, e.g., if the added noise to $P$ follows a uniform distribution as discussed in Section \ref{gaussian}.

\begin{lemma}[Mean estimation]
\label{lemma:mean_estimation}
Given a covariance matrix $\mSigma \in \gR^{k \times k}$, with $\lambda_{\min}$ being the minimum eigenvalue of $\mSigma$,
let $\PF = \{ \gN(\vmu,\mSigma) \mid \vmu \in \gR^k \}$ be a family of Gaussian distributions with unknown means. 
For any $P \in \gP$, let $\gQ \equiv \gQ(P,\beta) = \{Q \mid \TV(Q,P) \leq \beta \}$ and let $S$ represent n $\iid$ samples drawn from some noise-injected distribution $Q \in \gQ$.
Let $\hat{\theta} : S \rightarrow \gR^k$ be any mean estimator and let $d(\vmu_1,\vmu_2) = \NormII{\vmu_1-\vmu_2}$. 
We have:
\begin{align*}
  \inf_{\thetah}  \sup_{\substack{P \in \gP \\ Q \in \gQ}}  \E_{S \sim Q^n} \big[ d(\thetah(S), \theta(P)) \big]  \geq \frac{\sqrt{\lambda_{\min}}}{8\sqrt{n}} \left(\frac{1}{\sqrt{e}} + 2\delta \right), 
\end{align*}
where $0 \leq  \delta \leq \sup_{Q_1,Q_2\in \QF} \TV(P_1,Q_1) + \TV(P_2,Q_2)$, and $\TV$ is the total variation distance.
\end{lemma}

\begin{proof}
\refstepcounter{equation}
Consider any $P_1, P_2 \in \PF$ such that $P_1 = \Gauss{\vmu_1}{\mSigma}$ and $P_2 = \Gauss{\vmu_2}{\mSigma}$.
From Theorem \ref{col:adversarial_minimax} and Proposition \ref{lemma: min2KL} in Appendix \ref{app:detailed_proofs}, we have
\begin{align*}
&\inf_{\hat{\theta}}\sup_{P\in \PF, Q \in \QF} \E_{S \sim Q^n} \left[d(\hat{\theta}(S),\theta(P)) \right]\\
& \geq \frac{d(\theta(P_1),\theta(P_2))}{4} \left(\int_{S \in \gX^n} \min(p_1^n(S),p_2^n(S)) dS + \delta \right) \\
& \geq \frac{d(\theta(P_1),\theta(P_2))}{4} \left(\frac{1}{2} e^{-\KL(P_1^n||P_2^n)} + \delta \right) \\
& \geq \frac{d(\theta(P_1),\theta(P_2))}{4} \left(\frac{1}{2} e^{-n\KL(P_1||P_2)} + \delta \right). \tag{\theequation.a} \label{learning_by_kl}
\end{align*}
Let $\vv=\vmu_1-\vmu_2$. From Proposition \ref{lemma: KL between Gaussians}, we have :
\begin{align*}
\label{mean_b}
\KL(P_1 \Vert P_2) & = \frac{\vv^\top \mSigma^{-1} \vv}{2} \leq \frac{||\vv||_2^2}{2\lambda_{\min}}. \tag{\theequation.b} 
\end{align*}
Combining eq.\eqref{learning_by_kl} and eq.\eqref{mean_b}, we have 
\begin{align*}
\label{mean_c}
&\inf_{\hat{\theta}}\sup_{P\in \PF, Q \in \QF} \E_{S \sim Q^n} \left[d(\hat{\theta}(S),\theta(P))\right] \\
&\geq \frac{\NormII{\vv}}{8}\left(e^{-\frac{n \NormII{\vv}^2}{2\lambda_{\min}}} + 2\delta \right). \tag{\theequation.c} 
\end{align*}
Rewriting eq.\eqref{mean_c} as a function $f(u)$ where $u=\NormII{\vv}$, then $f$ reaches the maximum value when the derivative of $f(u)$ equals zero. Solving the previous equation, we get
\begin{align*}
\label{mean_d}
u = \sqrt{\frac{\lambda_{\min}}{n}}. \tag{\theequation.d} 
\end{align*}
Combining eq.\eqref{mean_c} and eq.\eqref{mean_d}, we conclude our proof.
\end{proof}
In Lemma \ref{lemma:mean_estimation}, the minimax rate is in the order of $\nicefrac{\delta}{\sqrt{n}}$, that is, if we control the growth rate of $\delta$ to be less than $\sqrt{n}$, then as the number of samples increases the lower bound tends to zero.
In our next lemma, we show the minimax rate for binary classification, for which we prescribe a generative model in order to describe the distribution of the observations.
\begin{lemma}[Binary classification]
\label{lemma:binary_classification}
Let $Y \in \{-1, +1\}$ be a Rademacher variable. 
Given a covariance matrix $\mSigma \in \gR^{k \times k}$, with $\lambda_{\min}$ being the minimum eigenvalue of $\mSigma$,
let $\vx$ follow a Gaussian distribution $\Gauss{Y\vw}{\mSigma}$ conditioned on $Y$ and with parameter $\vw \in \R^k$. 
Let $\gP = \left\{ P_\vw \mid p_\vw(\vx, Y), \forall \vw \right\}$ be a family of joint distributions over $(\vx, Y)$.
For any $P_\vw \in \gP$, let $\gQ \equiv \gQ(P_\vw, \beta) = \{Q \mid \TV(Q,P_\vw) \leq \beta \}$ and let $S$ represent n $\iid$ samples drawn from some noise-injected distribution $Q \in \gQ$.
Let $\hat{\theta} : S \rightarrow \gR^k$ be any estimator of $\vw$ and let $d(\vw_1,\vw_2)= \NormII{\vw_1-\vw_2}^2$. 
We have: 
\begin{align*} 
\inf_{\thetah} \sup_{\substack{P \in \gP \\ Q \in \gQ}} \E_{S \sim Q^n} \big[ d(\thetah(S), \theta(P)) \big]  \geq \frac{\lambda_{\min}}{8n} \left(\frac{1}{\sqrt{e}} + 2\delta \right), 
\end{align*}
where $0 \leq  \delta \leq \sup_{Q_1,Q_2\in \QF} \TV(P_1,Q_1) + \TV(P_2,Q_2)$, and $\TV$ is the total variation distance.
\end{lemma}
\begin{proof}
\refstepcounter{equation}
Let $(\vx_1 | Y_1=y_1) \sim \gN(y_1 \vw_1, \mSigma)$  and $(\vx_2 | Y_2 = y_2) \sim \gN(y_2 \vw_2, \mSigma)$ for any $\vw_1, \vw_2$.
We abuse a bit of notation and write the random variables inside the KL-divergence instead of distributions, we have:
\begin{align*}
\KL &\left( (\vx_1, Y_1) \Vert (\vx_2, Y_2) \right) \\
 &= \KL(Y_1 \Vert Y_2) + \KL(\vx_1 | Y_1 \Vert \vx_2 | Y_2) \\
 &= \frac{ (\vw_1-\vw_2)^\top \mSigma^{-1} (\vw_1-\vw_2)}{2}\\
 &\leq \frac{\NormII{\vw_1-\vw_2}^2}{2\lambda_{\min}} \tag{\theequation.a} \label{binary_a}
\end{align*}
Let $v=\NormII{\vw_1-\vw_2}^2$, combining eq.\eqref{learning_by_kl} and \eqref{binary_a}, we have:
\begin{align*}
 \inf_{\hat{\theta}}\sup_{\substack{P \in \gP \\ Q \in \gQ}} \E_{S \sim Q} \left[ d(\hat{\theta}(S),\theta(P))\right] \geq \frac{v}{8}\left(e^{-\frac{nv}{2\lambda_{\min}}}+2\delta\right). \tag{\theequation.b} \label{binary_b}
\end{align*}
Similar to eq.\eqref{mean_d}, eq.\eqref{binary_b} achieves maximum when $v=\frac{\lambda_{\min}}{n}$. Replacing this value into eq.\eqref{binary_b}, we derive the lemma.
\end{proof}

In Lemma \ref{lemma:binary_classification}, the minimax rate is in the order of $\nicefrac{\delta}{n}$, that is, if we control the growth rate of $\delta$ to be less than $n$, then as the number of samples increases the lower bound tends to zero.
In our next lemma, we show the minimax rate for the Procrustes analysis, for which we also prescribe a generative model in order to describe the distribution of the observations.
Procrustes analysis \citep{gower2004procrustes} is a widely used technique to transform one set of data to represent another set of data as closely as possible, typically in the field of shape analysis.
To the best of our knowledge, we are the first to characterize a lower bound for this type of analysis.

\begin{lemma}[Procrustes analysis]
\label{lemma:linear_regression}
 Consider the generative model $\vy = \mW \vx + \veps$,  where $\mW \in \gR^{k\times k}$ is the parameter, $\vx \sim\Gauss{\vzero}{\sigma^2 \mI}$ and $\veps \sim \Gauss{\vzero}{\epsilon^2 \mI}$.  
 Without loss of generality, assume $\mW \mW^\top = \mI$  and $\frac{\epsilon^2}{\sigma^2} \leq 4kn$. 
 Let $\gP = \left\{ P \mid p(\vx, \vy) \right\}$ be a family of joint distributions over $(\vx, \vy)$.
 For any $P \in \gP$, let $\gQ \equiv \gQ(P, \beta) = \{Q \mid \TV(Q,P) \leq \beta \}$ and let $S$ represent n $\iid$ samples drawn from some noise-injected distribution $Q \in \gQ$.
 Let $\hat{\theta} : S \rightarrow \gR^{k\times k}$ be any empirical estimator of $\mW$ and let $d$ be the square of Frobenius norm between two matrices $\mW_1, \mW_2$, that is, $d(\mW_1,\mW_2) = \NormII{\mW_1-\mW_2}^2$. 
 We have:
\begin{align*}
\inf_{\thetah} \sup_{\substack{P \in \gP \\ Q \in \gQ}} \E_{S \sim Q^n} \big[ d(\thetah(S), \theta(P)) \big]  \geq \frac{\epsilon^2}{8n\sigma^2} \left(\frac{1}{e} + 2\delta \right), 
\end{align*}
where $0 \leq  \delta \leq \sup_{Q_1,Q_2\in \QF} \TV(P_1,Q_1) + \TV(P_2,Q_2)$, and $\TV$ is the total variation distance.
\end{lemma}

\begin{proof}
\refstepcounter{equation}
From the model, we have that $(\vy | \vx) \sim \Gauss{\mW\vx}{\epsilon^2 \mI}$.
Let $P_1= p_1(\vx,\vy)$ with parameter $\mW_1$, $P_2=p_2(\vx,\vy)$ with parameter $\mW_2$.
Let $(\vx_1,\vy_1) \sim P_1$ and $(\vx_2,\vy_2) \sim P_2$.
We abuse a bit of notation and write the random variables inside the KL-divergence instead of distributions.
We have,
\begin{align*}
    \KL&\left( (\vx_1,\vy_1) \Vert (\vx_2,\vy_2) \right)  \\
    & = \KL(\vx_1 \Vert \vx_2) + \KL(\vy_1|\vx_1 \Vert \vy_2|\vx_2) \\    
    & =\int_\vx p(\vx) \KL(\vy_1|\vx \Vert \vy_2|\vx). 
\end{align*}
From Proposition \ref{lemma: KL between Gaussians} in Appendix \ref{app:detailed_proofs}, we have
\begin{align*}
    \int_\vx p(\vx) &\KL(\vy_1|\vx \Vert \vy_2|\vx)\\
    &= \int_x p(\vx) \frac{1}{\epsilon^2} \vx^\top (\mW_2-\mW_1)^\top (\mW_2-\mW_1)\vx \\
    &= \frac{1}{\epsilon^2} \E \left[\Tr( (\mW_2-\mW_1)^\top (\mW_2-\mW_1)\vx \vx^\top) \right] \\
    &= \frac{1}{\epsilon^2} \Tr \left((\mW_2-\mW_1)^\top(\mW_2-\mW_1) \E[\vx\vx^\top] \right)\\
    &= \frac{\sigma^2}{\epsilon^2} \Tr \left( (\mW_2-\mW_1)^\top (\mW_2-\mW_1) \right).  
\end{align*}
In addition, note that $d(\mW_1,\mW_2) = \NormII{\mW_1-\mW_2}^2 = \Tr\left((\mW_1-\mW_2)^\top (\mW_1-\mW_2) \right).$
Let $v = \Tr\left((\mW_1-\mW_2)^\top(\mW_1-\mW_2)\right)$. 
Since $\mW_1^\top\mW_2$ is also an orthogonal matrix, its eigenvalues are in $\{-1,+1\}$, thus, $-k \leq \Tr(\mW_1^\top \mW_2) \leq k$. 
Letting $v = 2k - 2\Tr(\mW_1^\top\mW_2)$, we have
\begin{align*}
    0 \leq v \leq 4k, \tag{\theequation.a} \label{regression_a}
\end{align*}
and
\begin{align*}
\inf_{\hat{\theta}}\sup_{P\in \PF, Q \in \QF} \hspace{-0.05in} \E_{S \sim Q} \left[d(\hat{\theta}(S),\theta(P))\right]  \geq \frac{v}{8}\left(e^{-\frac{n\sigma^2v}{\epsilon^2}}+2\delta\right). \tag{\theequation.b} \label{regression_b}
\end{align*}
Similar to eq.\eqref{mean_d}, eq.\eqref{regression_b} is maximized when $v=\frac{\epsilon^2}{n\sigma^2}$. Note that since $\frac{\epsilon^2}{n\sigma^2} \leq 4k$, this maximum is reachable. Replacing $v$ in eq.\eqref{regression_b}, we prove the result.
\end{proof}

\subsection{Types of adversarial noise}
\label{gaussian}

In this section, we show results on upper bounds that relate to $\delta$ for two types of noise, the multivariate Gaussian noise, and the multivariate uniform noise.
\begin{lemma}[Multivariate Gaussian noise]
\label{lemma:mulvar_gaussian_noise}
Let $P = \Gauss{\vmu}{\mSigma}$ be a Gaussian distribution with mean $\vmu \in \gR^k$ and covariance matrix $\mSigma \in \gR^{k\times k}$.
For a fixed $c \in \gR_+$, define the family of distributions $\gQ(P) = \{Q \mid Q = P + \Gauss{\vmu_{\vareps}}{\mSigma_{\vareps}}, \NormII{\vmu_{\vareps}} \leq c, \mSigma_{\vareps} \preceq \vmu_{\vareps} \vmu_{\vareps}^{\top} \}$.
Let $\lambda_{\min}$ denote the minimum eigenvalue of $\mSigma$.
We have,
\[
\sup_{Q \in \gQ(P)} \TV(P^n,Q^n) \leq \frac{c\sqrt{n}}{2\sqrt{\lambda_{\min}}}.
\]
\end{lemma}

\begin{proof}
 \refstepcounter{equation}
For any $Q \in \gQ(P)$, we have $Q = \Gauss{\vmu+\vmu_\vareps}{\mSigma+\mSigma_\vareps}$.
By using Pinsker's inequality we have, 
\begin{align*}
     \sup_{Q \in \QF(P)} \TV(P^n||Q^n) &\leq \sup_{Q \in \QF(P)} \sqrt{\frac{\KL(P^n||Q^n)}{2}} \\
     &\hspace{-0.2in} = \sqrt{\frac{n \; \sup_{Q \in \QF(P)} \KL(P||Q)}{2}} 
    \tag{\theequation.a} \label{eq: a}
\end{align*}
From Proposition \ref{lemma: KL between Gaussians} in Appendix \ref{app:detailed_proofs}, we have 
\begin{align*}
	& \KL(P\ \Vert\ Q) = \frac{1}{2} \Big[ \Tr \left( (\mSigma+\mSigma_\vareps)^{-1}\mSigma \right) \\
    &\hspace{0.2in}+ \vmu_\vareps^\top(\mSigma+\mSigma_\vareps)^{-1}\vmu_\vareps - k + \log{\frac{\det(\mSigma+\mSigma_\vareps)}{\det{\mSigma}}} \Big].
\end{align*}
Let $\mO$ be the $k\times k$ zero matrix.  
Rewrite $\KL(P\ \Vert\ Q)$ as a function $f(\mX)$, where
\begin{align*}
 f(\mX) &=  \frac{1}{2} \Big[ \Tr \left((\mX+\mSigma)^{-1}\mSigma\right) + \vmu_\vareps^\top(\mX+\mSigma)^{-1} \vmu_\vareps \\
 &\hspace{0.3in} -k+\log{\frac{\det{(\mX+\mSigma)}}{\det{\mSigma}}} \Big],
\end{align*}
for all $\mSigma_\vareps$, $t\in[0,1]$, let $g(t)=f(\mO + t\mSigma_\vareps)$, we would like to prove $g(t)$ reaches maximum when $t=0$.
The derivative of $g(t)$ is 
\begin{align*}
\frac{\partial{g(t)}}{\partial{t}} 
& \hspace{-0.025in} = \hspace{-0.025in} \frac{1}{2} \hspace{-0.025in} \Tr [ \mSigma_\vareps (\mSigma+\mSigma_\vareps)^{-1}(\mSigma_\vareps-\vmu_\vareps\vmu_\vareps^\top)(\mSigma+\mSigma_\vareps)^{-1} ]. 
\end{align*}
We have $(\mSigma+\mSigma_\vareps)^{-1} \succeq 0$, $\mSigma_\vareps \succeq 0$. Because $Q \in \QF(P)$, we have $\mSigma_\vareps \preceq \vmu_\vareps\vmu_\vareps^\top$. Therefore the matrix $\mSigma_\vareps (\mSigma+\mSigma_\vareps)^{-1}(\mSigma_\vareps-\vmu_d\vmu_d^\top)(\mSigma+\mSigma_\vareps)^{-1} \preceq 0$ and $\frac{\partial{g(t)}}{\partial{t}} \leq 0$.
Note that:
\begin{align*}
 g(1) & = \KL(P\ \Vert\ Q), \\
 g(0) & =f(\mO), \\
 g(1) & =g(0)+g'(v)\cdot 1, \ v \in [0,1] \\
  & \leq g(0).
\end{align*}
Thus, we have:
\begin{align*}
&\forall Q \in \QF(P),\; \KL(P||Q) = f(\mSigma_\vareps) = g(1) \\
&\hspace{-0.1in}\leq g(0) = f(\mO)\leq \frac{1}{2}(\vmu_\vareps^\top\mSigma^{-1}\vmu_\vareps) \leq \frac{c^2}{2 \lambda_{\min}}. \tag{\theequation.b} \label{eq: b}
\end{align*}
From eq.\eqref{eq: a} and eq.\eqref{eq: b}, we finish the proof.
\end{proof}

In Lemma \ref{lemma:mulvar_gaussian_noise}, we analyze the case of $P$ being a Gaussian distribution which is perturbed by another Gaussian noise.
In that particular case, the noise-injected distribution also follows a Gaussian distribution, although, with different mean and covariance.
In Section \ref{discussion}, we will discuss how to control $c$ in order to guarantee a good lower bound.
Next, we show a perhaps more interesting case, where $P$ also follows a Gaussian distribution but is now poisoned by a uniform distribution, which results in $Q$ being different from a Gaussian distribution.

\begin{lemma}[Multivariate Uniform Noise]
\label{lemma:mulvar_uniform_noise_2}
Let $P = \Gauss{\vmu}{\mSigma}$ be a Gaussian distribution with mean $\vmu \in \gR^k$ and covariance matrix $\mSigma \in \gR^{k\times k}$.
For a fixed $c \in \gR_+$, define the family of distributions $\gQ(P) = \{Q \mid Q = P + \gU^k(-\vareps,+\vareps), \vareps \leq c \}$, where $\gU^k$ denotes the $k$-dimensional uniform distribution.
We have,
\begin{align*}
\sup_{Q \in \gQ(P)} \hspace{-0.07in} \TV(P^n,Q^n) \leq & \sqrt{ c \sum_{i=1}^k \hspace{-0.025in} \frac{n}{\sqrt{ 2\pi \emSigma_{ii}}}  + \hspace{-0.025in} \frac{nc^2}{4}\Tr(\mB\mSigma^{-1}) },
\end{align*}
where $\emB_{ij} = \frac{2}{\pi} \arctan \emV_{ij}$ if $\emSigma_{ij} > 0$  and $\emB_{ij} = \frac{1}{4\pi} \left( 6 \arctan{\emV_{ij}} - 2 \arccot{\emV_{ij}} + \pi \right)$ otherwise, for $\emV_{ij} = \frac{\emSigma_{ij}}{\sqrt{\emSigma_{ii}\emSigma_{ij}-{\emSigma_{ij}}^2}}$.
\end{lemma}
\begin{proof} 
\refstepcounter{equation}
$\forall P \in \PF,Q \in \QF(P)$, let $q(\cdot)$ and $p(\cdot)$ be the density function for distributions $P$ and $Q$. Let $\bar{P} = \bar{p}(\cdot)$ where $\bar{p}(\vx)=p(\vx+\vmu)$ denotes the centered distribution of $P$, and let $\bar{Q} = \bar{q}(\cdot)$ where $\bar{q}(\vx)=q(\vx+\vmu)$ is the respective centered $Q$ and let $u(\cdot)$ be the density of $\gU^k$.
By convolution, we have:
\begin{align*}
    \bar{q}(\vx) & = \int_\vy \bar{p}(\vx-\vy)u(\vy) =\frac{1}{(2 \vareps)^k}\int_{\vy \in [-\vareps,\vareps]^k} \bar{p}(\vx-\vy) \\
    & = \bar{p}(\vx + \vv) , \text{\ for some\ } \vv \in [-\vareps,\vareps]^k \tag{\theequation.a} \label{proof:mulvar_uni_a}\\ 
    & \geq \bar{p}(\vx +  \vareps \sign \vx ). \tag{\theequation.b} \label{proof:mulvar_uni_b}
\end{align*}
Note that eq.\eqref{proof:mulvar_uni_a} follows from iteratively applying the Mean Value Theorem for each dimension and eq.\eqref{proof:mulvar_uni_b} follows from the fact that $\bar{P}$ is centered, the density function achieves the maximum value at the origin.
Next, we bound the KL divergence between $P$ and $Q$.
For any $Q \in \gQ(P)$,
\begin{align*}
&\KL(P||Q)  = \KL(\bar{P}||\bar{Q}) = \E_{\vx \sim \bar{P}}\left[\log\frac{\bar{p}(\vx)}{\bar{q}(\vx)}\right] \\
&\leq \E_{\vx \sim \bar{P}}\left[\log\frac{\bar{p}(\vx)}{\bar{p}(\vx+\epsilon \sign \vx)}\right]\\
& \leq \frac{1}{2}\E_{\vx \sim \bar{P}} \Big[(\vx+ \varepsilon \sign \vx)^\top \mSigma^{-1} (\vx+ \varepsilon \sign \vx) \\ 
&\hspace{0.7in} - \vx^\top \mSigma^{-1} \vx  \Big] \\
& \leq \frac{1}{2}\E_{\vx \sim \bar{P}} \left[ 2 \varepsilon \sign \vx^\top \mSigma^{-1} \vx + \varepsilon^2 \sign \vx^\top \mSigma^{-1} \vx \right] \\
& = \frac{1}{2} \E_{\vx \sim \bar{P}} \Big[ 2\varepsilon \Tr \left( \vx \sign \vx\top \mSigma^{-1}\right)  \\
&\hspace{0.7in} + \varepsilon^2 \Tr \left(\sign \vx \sign \vx^\top \mSigma^{-1}  \right) \Big] \\
& = \frac{1}{2} \Big[ 2 \varepsilon \Tr \left( \E_{\vx \sim \bar{P}} \left[ \vx \sign \vx^\top \right] \mSigma^{-1}\right) \\
& \hspace{0.3in} + \varepsilon^2 \Tr \left(\E_{\vx \sim \bar{P}} \left[ \sign \vx \sign \vx^\top \right] \mSigma^{-1}\right) \Big] 
\end{align*}
\begin{align*}
& = \varepsilon \Tr \left( \mA \mSigma^{-1} \right) + \frac{\varepsilon^2}{2} \Tr \left(  \mB \mSigma^{-1} \right),  \tag{\theequation.c} \label{proof:mulvar_uni_c}
\end{align*}
where $\emA_{ij} = \frac{\sqrt{2}\emSigma_{ij}}{\sqrt{\pi \emSigma_{ii}}}$, and $\emB_{ij} = \frac{2}{\pi} \arctan \emV_{ij}$ if $\emSigma_{ij} > 0$, and $\emB_{ij} = \frac{1}{4\pi} \left( 6 \arctan{\emV_{ij}} - 2 \arccot{\emV_{ij}} + \pi \right)$ otherwise, for $\emV_{ij} = \frac{\emSigma_{ij}}{\sqrt{\emSigma_{ii}\emSigma_{ij}-{\emSigma_{ij}}^2}}$.
Step \eqref{proof:mulvar_uni_c} is the direct result of applying Proposition \ref{lemma: Covaraince of sign of gaussian} (Appendix \ref{app:detailed_proofs}) for each dimension.
Also, note that $\mB \succeq 0$ and $\mSigma^{-1} \succeq 0$.
Therefore, $\Tr(\mB\mSigma^{-1}) \geq 0$. 
We also rewrite $\mA$ as $\mC\mSigma$, where $\mC$ is a diagonal matrix with $\emC_{ii} = \frac{\sqrt{2}}{\sqrt{\pi \emSigma_{ii}}}$. 
Then we have, $\Tr(\mA\mSigma^{-1}) = \Tr(\mC\mSigma \mSigma^{-1}) = \sum_{i=1}^k \frac{\sqrt{2}}{\sqrt{\pi \emSigma_{ii}}}$.

Finally, eq.\eqref{proof:mulvar_uni_c} is a quadratic function of $\varepsilon$ with positive coefficients and reaches its maximum when $\varepsilon = c$.
\begin{align*}
	\varepsilon \Tr \left( \mA \mSigma^{-1} \right) &+ \frac{\varepsilon^2}{2} \Tr \left(  \mB \mSigma^{-1} \right) \\
	&= \varepsilon \sum_{i=1}^k \frac{\sqrt{2}}{\sqrt{\pi \emSigma_{ii}}}  + \frac{\varepsilon^2}{2} \Tr(\mB\mSigma^{-1}) \\
	&\leq c \sum_{i=1}^k \frac{\sqrt{2}}{\sqrt{\pi \emSigma_{ii}}}  + \frac{c^2}{2}\Tr(\mB\mSigma^{-1}).
\end{align*}
Combining eq.\eqref{proof:mulvar_uni_c} and  Pinsker's inequality, we conclude the proof.
\end{proof}

\section{Discussion}
\label{discussion}
We first note that the parameter $\beta$ in Theorem \ref{col:adversarial_minimax} should be small\footnote{How small the value should be will depend on the particular application.}, that is, a large value of $\beta$ will allow the adversary to largely perturb the original distribution which would make it certainly easy to detect that we are observing data from a poisoned distribution.
In Section \ref{gaussian}, we note that the upper bounds on the total variation distances grow with respect to $n$. 
Therefore, to obtain non-vacuous upper bounds in Lemmas $\ref{lemma:mulvar_gaussian_noise}$ and $\ref{lemma:mulvar_uniform_noise_2}$ we choose values of $c$ to control the magnitude and make them small.

For the multivariate Gaussian noise we have the following corollary.
\begin{corollary}
In Lemma \ref{lemma:mulvar_gaussian_noise}, for any $t \in [0,1]$, if we set $c$ as follows:
\[
c= \frac{2t\sqrt{\lambda_{\min}}}{\sqrt{n}},
\]    
then we have that $\sup_{Q \in \gQ(P)} \TV(P^n,Q^n) \leq t.$
\end{corollary}

For the multivariate uniform noise we have the following corollary.
\begin{corollary}
In Lemma \ref{lemma:mulvar_uniform_noise_2}, for any $t \in [0,1]$, if we set $c$ as follows:
\begin{align*}
&c = \frac{1}{n\Tr(\mB\mSigma^{-1})} \Bigg[-n\sum_{i=1}^k\frac{\sqrt{2}}{\sqrt{\pi\emSigma_{ii}}} \\
&\hspace{0.4in} + \Big[ n^2 \Big(\sum_{i=1}^k\frac{\sqrt{2}}{\sqrt{\pi\emSigma_{ii}}} \Big)^2 + 4 n \Tr(\mB\mSigma^{-1}) t^2 \Big]^{\frac{1}{2}} \Bigg],
\end{align*}
then we have that $\sup_{Q \in \gQ(P)} \TV(P^n,Q^n) \leq t.$
\end{corollary}

In both cases, the bounds above are now constant with respect to a real number $t$.
For instance, if one sets $t=0.01$, then, as argued in Remark \ref{remark:bound}, one necessarily ``pays'' \textit{at most} $0.02$ extra in the minimax risk with respect to the case where there is no adversary.

\section{Concluding Remarks}
In this paper, we look into the statistical limits of learning from noise-injected distributions.
We adapted Le Cam's lemma for our specific setting and showed through examples that our results can be applied to different tasks and noise-injection attacks due to the decomposability of the bound into the standard Le Cam bound and the adversarial term, $\delta$.

Combining our results from Sections \ref{learning_probs} and \ref{gaussian}, we provided 6 different results (Table \ref{tab:summary}) for our adversarial setting.
As future work, one possible line is to analyze the regime in which the data comes from a certain parameterized distribution $P$ and the added noise is Gaussian.
In Lemma \ref{lemma:mulvar_gaussian_noise}, as a first set of results, we analyzed the case in which $P$ and the noise are Gaussians.
Hence, for $P$ different than Gaussian the question remains open.

\bibliography{reference}
\bibliographystyle{agsm}

\clearpage
\appendix
\onecolumn
\def\toptitlebar{
	\hrule height4pt
	\vskip .25in
}
\def\bottomtitlebar{
	\vskip .25in
	\hrule height1pt
	\vskip .25in
}
\thispagestyle{empty}
\hsize\textwidth
\linewidth\hsize \toptitlebar {\centering
	{\large\bf SUPPLEMENTARY MATERIAL \\ A Le Cam Type Bound for Adversarial Learning and Applications \par}
}
\vspace{-0.1in} \bottomtitlebar


\section{Propositions}
\label{app:detailed_proofs}

The following statements are well-known results and are included for clarity purposes. 

\begin{proposition}[\cite{Tsybakov09}]
\label{lemma: min2KL}
Let $X$ be a random variable with support on $\gX$. For any two distributions $P_1 = p_1(\cdot)$ and $P_2 = p_2(\cdot)$, we have: 
\begin{align*}
    \int_{x \in \gX} \min(p_1(x),p_2(x)) dx \geq \frac{1}{2} e^{-\KL(P_1 \Vert P_2)}.
\end{align*}
\end{proposition}

\begin{proposition}[\cite{duchi2007derivations}]
Given two $d$-dimensional normal distributions $D_1 = \Gauss{\vmu_1}{\mSigma_1}$ and $D_2 = \Gauss{\vmu_2}{\mSigma_2}$, we have:
\begin{align*}
    \KL(D_1 \Vert D_2) =  \frac{1}{2} \Big[\Tr(\mSigma_2^{-1}\mSigma_1) +(\vmu_2-\vmu_1)^\top \mSigma_2^{-1}(\vmu_2-\vmu_1) 
    - d + \log \frac{|\mSigma_2|}{|\mSigma_1|} \Big].
\end{align*}
\label{lemma: KL between Gaussians}
\end{proposition}

\begin{proposition}[\cite{eaton1983multivariate}]
\label{lemma: Conditions of Gaussian}
Given two random vectors $\vx_1,\vx_2$ that jointly follow a Gaussian distribution, that is, $\begin{bmatrix}
           \vx_1 \\
           \vx_2 
         \end{bmatrix} \sim \Gauss{\begin{bmatrix}
           \vmu_1 \\
           \vmu_2 
         \end{bmatrix}} {
         \begin{bmatrix}
           \mSigma_{11} & \mSigma_{12} \\
           \mSigma_{21} & \mSigma_{22}
\end{bmatrix} }$, the conditional distribution of $\vx_2$ given $\vx_1$ is:
\begin{align*}
 (\vx_2 \mid \vx_1=\va) \sim  \gN \Big( \vmu_2 + \mSigma_{21} \mSigma_{11}^{-1}(\va-\vmu_1), \mSigma_{22} -\mSigma_{21}\mSigma_{11}^{-1}\mSigma_{12} \Big).
\end{align*}
\end{proposition}

\begin{proposition} \label{lemma: Covaraince of sign of gaussian}
Given two univariate random variables $X_1,X_2$ that jointly follow a zero mean Gaussian distribution, that is, 
$ (X_1, X_2) \sim 
\Gauss{(0,0)}
{
	\begin{bmatrix} \emSigma_{11} & \emSigma_{12} \\ \emSigma_{21} & \emSigma_{22} \end{bmatrix}
},$
the covariance between $\sign X_1$ and $X_2$ is:
\begin{align*}
	\Cov(\sign X_1, X_2) = \frac{\sqrt{2}\emSigma_{12}}{\sqrt{\pi \emSigma_{11}}},
\end{align*}
and let $v=\frac{\emSigma_{12}}{\sqrt{\emSigma_{11}\emSigma_{12}-{\emSigma_{12}}^2}}$, then the covariance between $\sign X_1$ and $\sign X_2$ is:
\begin{align*}
	\Cov(\sign X_1, \sign X_2) = \begin{cases}
	\frac{2}{\pi} \arctan v, & \mathrm{\;if\;} \emSigma_{12} > 0, \\
	\frac{1}{4\pi} \left( 6 \arctan{v} - 2 \arccot{v} + \pi \right), & \mathrm{\;otherwise\;}.
	\end{cases}
\end{align*}
\end{proposition}

\begin{proof}

Let $p(X_1, X_2)$ be the PDF of the Gaussian distribution as defined above.

For $\Cov(\sign(X_1), X_2) = \E[\sign(X_1) X_2] - \E[\sign(X_1)] \E[X_2]$.
We know that $\E[\sign(X_1)] = 0$ and $\E[X_2] = 0.$
Also, $\E[\sign(X_1) X_2] =
\int_{X_1 \in [0,\infty), X_2 \in \gR} X_2 p(X_1, X_2)+
\int_{X_1 \in (-\infty,0], X_2 \in \gR} -X_2 p(X_1, X_2)$.

For $\Cov(\sign(X_1), \sign(X_2)) = \E[\sign(X_1) \sign(X_2)] - \E[\sign(X_1)] \E[\sign(X_2)].$
We know that $\E[\sign(X_1)] = 0$ and $\E[\sign(X_2)] = 0.$
Also, $\E[\sign(X1) \sign(X2)] =
\int_{X_1 \in [0,\infty), X_2 \in [0,\infty)} p(X_1, X_2) +
\int_{X_1 \in [0,\infty), X_2 \in (-\infty,0]} -p(X_1, X_2) +
\int_{X_1 \in (-\infty,0], X_2 \in [0,\infty)} -p(X_1, X_2) +
\int_{X_1 \in (-\infty,0], X_2 \in (-\infty,0]} p(X_1, X_2).$

Working out the above integrals lead to the desired result.
\end{proof}

\end{document}